\documentclass[letterpaper, 10 pt, journal, twoside]{IEEEtran}
\usepackage{etoolbox}
\makeatletter
\patchcmd{\@makecaption}
  {\scshape}
  {}
  {}
  {}
\makeatother
\usepackage{amsmath,amsfonts}
\usepackage{cite}
\usepackage{array}
\usepackage[caption=false,font=footnotesize,labelfont=sf,textfont=sf]{subfig}
\usepackage{textcomp}
\usepackage{stfloats}
\usepackage{url}
\usepackage{verbatim}
\usepackage{graphicx}
\def\BibTeX{{\rm B\kern-.05em{\sc i\kern-.025em b}\kern-.08em
    T\kern-.1667em\lower.7ex\hbox{E}\kern-.125emX}}
\usepackage{balance}
\usepackage{import}
\DeclareMathOperator*{\argmin}{arg\,min}

\usepackage{amssymb}
\usepackage[linesnumbered]{algorithm2e}\RestyleAlgo{ruled}
\SetInd{0.2em}{0.4em}
\usepackage{arydshln}
\usepackage[nocomma]{optidef}
\usepackage{amsthm}
\theoremstyle{plain}
\newtheorem{theorem}{Theorem}[section]

\newcommand{\newtext}[1]{#1}
\newcommand{\oldtext}[1]{}

\usepackage{bbm}
\usepackage{multirow}
\usepackage[T1]{fontenc}
\usepackage{xcolor}
\usepackage{soul}
\usepackage{ulem}
\usepackage{mathtools}
\newtagform{blue}{\color{blue}(}{)}
\newtagform{black}{\color{black}(}{)}

\begin{document}
\title{CACTO: Continuous Actor-Critic with Trajectory Optimization---Towards global optimality}
\author{Gianluigi Grandesso$^{1}$, \newtext{Elisa Alboni$^{2}$}, Gastone P. Rosati Papini\oldtext{$^{2}$}\newtext{$^{3}$, \textit{Member, IEEE}}, Patrick M. Wensing\oldtext{$^{3}$}\newtext{$^{4}$, \textit{Member, IEEE}} and Andrea Del Prete\oldtext{$^{4}$}\newtext{$^{5}$, \textit{Member, IEEE}}
\thanks{Manuscript received: November 10th, 2022; Revised: February 15th, 2023; Accepted: March 23rd, 2023.}
\thanks{This paper was recommended for publication by
Editor Lucia Pallottino upon evaluation of the Associate Editor and Reviewers’
comments. This work was supported by the Italian MUR through the ``Departments of Excellence'' programme and by PRIN project DOCEAT (CUP n. E63C22000410001).}
\thanks{$^{1}$, $^{2}$\newtext{ \oldtext{and $^{4}$}}\newtext{, $^{3}$ and $^{5}$} are with the Dept. of Industrial Engineering, University of Trento, Italy
        [{\tt\footnotesize gianluigi.grandesso; \newtext{\oldtext{\tt\footnotesize andrea.delprete; \tt\footnotesize gastone.rosatipapini}}\newtext{ \tt\footnotesize elisa.alboni; \tt\footnotesize gastone.rosatipapini;
        \tt\footnotesize andrea.delprete}}]\tt\footnotesize@unitn.it}
\thanks{\oldtext{\newtext{$^{3}$}}\newtext{$^{4}$} is with the Dept. of Aerospace and Mechanical Engineering, University of Notre Dame, Indiana, USA
        {\tt\footnotesize pwensing@nd.edu}}
\thanks{Digital Object Identifier (DOI): see top of this page.}
\thanks{\textcopyright 2023 IEEE.  Personal use of this material is permitted.  Permission from IEEE must be obtained for all other uses, in any current or future media, including reprinting/republishing this material for advertising or promotional purposes, creating new collective works, for resale or redistribution to servers or lists, or reuse of any copyrighted component of this work in other works.
}
}
\markboth{IEEE Robotics and Automation Letters. Postprint Version. Accepted April, 2023}
{Grandesso \MakeLowercase{\textit{et al.}}: CACTO: Continuous Actor-Critic with Trajectory Optimization---Towards global optimality} 

\maketitle
\normalem
\begin{abstract}
This paper presents a novel algorithm for the continuous control of dynamical systems that combines Trajectory Optimization (TO) and Reinforcement Learning (RL) in a single framework. The motivations behind this algorithm are the two main limitations of TO and RL when applied to continuous nonlinear systems to minimize a non-convex cost function. Specifically, TO can get stuck in poor local minima when the search is not initialized close to a ``good'' minimum. On the other hand, when dealing with continuous state and control spaces, the RL training process may be excessively long and strongly dependent on the exploration strategy.
Thus, our algorithm learns a ``good'' control policy via TO-guided RL policy search that, when used as initial guess provider for TO, makes the trajectory optimization process less prone to converge to poor local optima.
Our method is validated on several reaching problems featuring non-convex obstacle avoidance with different dynamical systems, including a car model with 6\oldtext{d}\newtext{D} state, and a 3-joint planar manipulator. Our results show the great capabilities of CACTO in escaping local minima, while being more computationally efficient than the \newtext{Deep Deterministic Policy Gradient (}DDPG\newtext{)}\newtext{ and Proximal Policy Optimization (PPO)} RL algorithm\newtext{s}.
\end{abstract}

\vspace*{-0.01cm}
\section{Introduction}

When a model of the system to be controlled is available, one of the most common and flexible techniques to compute optimal trajectories is gradient-based Trajectory Optimization (TO). Starting from a specific initial state, TO can find the control sequence that minimizes a cost function representing the task to be accomplished, where the system's dynamics and possible state and control limits are considered as constraints. Such a powerful framework has led to excellent results when the problem is convex or slightly non-convex, especially when used in a Model Predictive Control (MPC) fashion. For example, it has been successfully employed to control high-dimensional nonlinear systems such as quadrupeds and humanoids~\cite{miniCheetah,mpc_quadruped,humanoid1,humanoid2}. However, when the task to be accomplished requires a highly non-convex cost function and/or the dynamics is highly nonlinear, the presence of multiple local minima, some of which are of poor quality (i.e., associated to a cost that is significantly worse than the global minimum), often prevents TO from finding a satisfying control trajectory. A possible solution to this problem is providing TO with a good initial guess, which turns out to be very complex---not to say impossible---without a deep prior knowledge about the system, which is often unavailable in practice. The IREPA algorithm~\cite{irepa} tackles this problem by building a kinodynamic Probabilistic Road Map (PRM) and approximating the \emph{Value function} and control policy, which is then used to warm-start an MPC. The method produced satisfying results on a 3-DoF system, but it is limited to problems with a fixed terminal state, and scaling to high dimensions seems not trivial due to its need to explicitly store the locally optimal trajectories in the PRM edges.
 
Approaches that can find the global optimum of non-convex problems exist, and are based on the Hamilton-Jacobi-Bellman equation~\cite{BardiMartino1997} (for continuous-time problems) or Dynamic Programming (for discrete-time problems). However, these methods suffer from the curse of dimensionality, which restricts their applicability to systems with extremely few degrees of freedom. Some efficient solutions exist, but for specific problems, such as simple integrator dynamics with control-independent cost~\cite{Tsitsiklis1995, Polymenakos1998}. Alternatively, the tensor-train decomposition~\cite{Gorodetsky2015} has been used to reduce the computational complexity of \emph{value iteration}, by representing the \emph{Value function} in a compressed form. However, the approach has been tested on stochastic optimal control problems with at most 7-dimensional state and scalar control. Subsequent improvements of that approach~\cite{Stefansson2016} could solve problems with a 12-dimensional state, but required some restrictions on the form of the dynamics and cost.

On the other hand, in the recent years deep RL has shown impressive results on continuous state and control spaces, which represented its greatest challenge until 2015, when the ground-breaking algorithm \oldtext{Deep Deterministic Policy Gradient (DDPG)}\newtext{DDPG} \cite{ddpg} was presented. Successively, many variants have been developed to further improve its performance, such as TD3, SAC, and RTD3~\cite{td3,sac,rtd3}. In 2021, Zhang \textit{et al.}~\cite{ae-ddpg} managed to speed up DDPG's training time and enhance the effectiveness of its learning with their expansion called AE-DDPG, by making the agent latch on ``good'' trajectories very soon through the use of asynchronous episodic control and improving exploration with a new type of control noise. However, deep RL is intrinsically limited by \oldtext{the very long time needed to explore the state-control space}\newtext{its low sample efficiency, which implies the need for a considerable number of interactions with the environment to reach a good performance level}. 

To mitigate this problem, Levine \textit{et al.} \cite{gps} proposed to guide the exploration process by using Differential Dynamic Programming (DDP) as a generator of guiding samples that push the policy search towards low-cost regions. However, the \emph{imitation} component of this approach makes its capability to find an optimal control policy strongly dependent on the quality of the guiding samples. This downside applies also to the method proposed by Mordatch and Todorov~\cite{mordatch} combining policy learning and TO through the Alternating Direction Method of Multipliers (ADMM), which involves imitation in that the policy learning problem is reduced to a sequence of trajectory optimization and regression problems. In general, all the methods that are \emph{imitation}-oriented (\textit{e.g.} \cite{DDPGfD,PsiPhi,guidedTO}) suffer from the same limitation: the quality of what the RL algorithm can learn is limited by the quality of the demonstrations or guiding trajectories found by TO.

In a similar spirit to~\cite{gps}, but with an imitation-free approach, our algorithm aims to mitigate both problems: the local minima issue affecting TO, and the \oldtext{long training time}\newtext{low sample efficiency} of RL. Our main contribution is an algorithm named CACTO (Continuous Actor-Critic with TO). The algorithm combines TO and RL in such a way that their interplay guides the search towards the globally optimal control policy. Our main contributions are:
\begin{enumerate}
    \item We present a novel RL algorithm that exploits TO to speed up the search, and that (contrary to previous work) does not rely on imitation.
    \item Our tests show that initializing TO with the CACTO policy outperforms standard warm-starting techniques. 
    \item We have proved that, considering a discrete-space version of CACTO with look-up tables instead of deep neural networks (DNN), the policy approaches global optimality 
    as the algorithm proceeds.
\end{enumerate}



\vspace{-5pt}
\section{Method}


This section presents an optimization algorithm to solve a finite-horizon discrete-time \oldtext{OCP}\newtext{optimal control problem (OCP)} that takes the following general form: 
\begin{mini!}|l|[2]<b>
{\scriptstyle{X, U}}
{J(X,U) = \sum_{k=0}^{T \newtext{-1}} l_k\left(x_k,u_k\right) \newtext{+ \ l_T\left(x_T\right)} 
\label{OCP_objective}}
{\label{OCP_optimizationProblem}}
{}
\addConstraint{x_{k+1}}{= f_k(x_k,u_k)  \quad \forall k=0\dots T-1\label{OCP_dyn_const}}
\addConstraint{|u_k|}{\in \mathcal{U} \quad\quad\quad\quad \, \ \,  \forall k=0\dots T-1\label{OCP_path_const}}
\addConstraint{x_0}{= x_{init}\label{OCP_ICS}}
\end{mini!}
where the state and control sequences $X=x_{0\dots T}, U=u_{0\dots T-1}$, with $x_k \in \mathbb{R}^{n}$ and $u_k \in \mathbb{R}^{m}$, are the decision variables. The cost function $J(\cdot)$ is defined as the sum of the running costs $l_k\left(x_k,u_k\right)$ and the terminal cost $l_T(x_T\oldtext{,\cdot})$. The dynamics, control limits and initial conditions are represented by \eqref{OCP_dyn_const}, \eqref{OCP_path_const} and \eqref{OCP_ICS}\newtext{, with $\mathcal{U} = \{u \in \mathbb{R}^m: |u| \le u_{max} \}$}.
\begin{figure}[t]
    \includegraphics[width = \columnwidth]{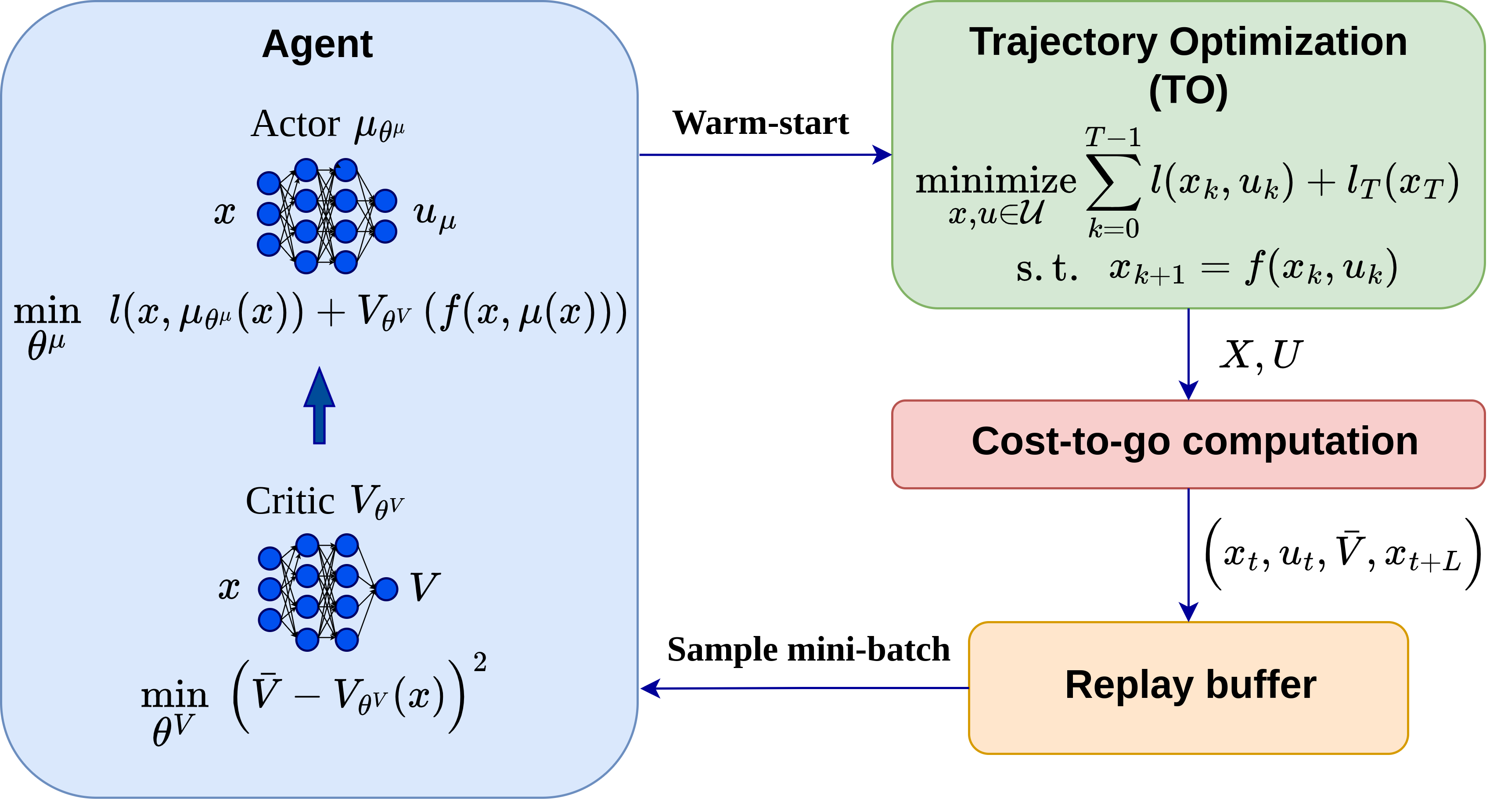} 
    \caption{Scheme of CACTO: at each episode, a TO problem is warm-started with a policy rollout and solved, then we compute the cost-to-go associated with each state of the trajectory found by TO and store the related transition in the replay buffer. Finally, we update critic and actor by sampling a mini-batch of transitions from the replay buffer.}
    \label{CACTO_scheme}
    \vspace{-8pt}
\end{figure}
%

Our algorithm combines ideas from RL and TO. The core idea is to exploit TO to guide the exploration towards low-cost regions, \newtext{\oldtext{so as to speed up the learning process}}\newtext{so as to make the learning process more efficient}. In turn, a rollout of the currently learned policy is used to initialize the next TO problem. In this way, the RL and TO components of the algorithm help each other, making convergence faster: as training proceeds, the TO solver is provided with better initial guesses and this increases the chance of obtaining better solutions; in turn, these trajectories drive the agent along lower-cost paths, pushing the critic towards a better approximation of the optimal \emph{Value function} and, consequently, the actor towards the optimal policy. Fig.~\ref{CACTO_scheme} shows a scheme of CACTO.

\vspace{-8pt}
\subsection{\newtext{TO and RL} Notation}

Besides the control sequence $[u_0, \dots, u_{T-1}]$, we also make use of the term ``\textit{policy}'' $\pi(x)$, that belongs to the RL community and refers to the mapping from states $x$ to controls $u$ (``actions'' in the RL language).
%
\oldtext{\newtext{Another fundamental concept is that of the \emph{Value function}}.}In the optimal control (OC) community, the \emph{Value function} of a policy $\pi$ describes the cost-to-go following $\pi$ from $x$ at time $t$, namely:
\begin{equation*}
V_t^{\pi}(x_t)\oldtext{^{OC}} = \sum_{k=t}^{T \newtext{-1}}l_{k}(x_k,\pi(x_k)) \newtext{\ + \ l_T\left(x_T\right)}
\end{equation*}
%
In RL, the \emph{Value function} is the sum of the rewards starting from a state $x$ at time $t$ and following a policy $\pi$; therefore, we aim to maximize it, opposite to what we do in OC. For reasons of clarity and to keep the notation simple, from now on we always refer to the \emph{Value function} $V(x_t)$ in an OC sense.
The same reasoning applies to the \emph{Q-value function} $Q(x_t,u_t)$, which, in RL, is given by the sum of the rewards after applying $u$ from $x$ at time $t$ and following $\pi$ thereafter, whereas in the OC community, it represents the cost-to-go after performing $u_t$ from $x_t$ and following $\pi$ thereafter:
\begin{equation*}
Q_t^{\pi}(x_t, u_t)\oldtext{^{OC}} = l_t(x_t,u_t)+\sum_{k=t+1}^{T \newtext{-1}}l_{k}(x_k,\pi(x_k)) \newtext{\ + \ l_T\left(x_T\right)}
\end{equation*}
%
In the following, to simplify the notation, we consider time as the last component of the state vector ($x[n]=t$), so that we can drop the explicit time dependency for the  \emph{Value function}, policy, dynamics, and running cost.

\vspace{-8pt}
\subsection{Algorithm Description}
As shown in Fig.~\ref{CACTO_scheme}, our algorithm can be broken down into four phases. In the TO phase (green block), the optimal state and control trajectories are computed considering a random initial state. Then, the cost-to-go is computed (red block) for each state of the optimal trajectory and stored in a buffer (orange block). This buffer is then sampled to update both the critic and actor (blue block). Finally, to close the loop, a rollout of the actor is used to warm-start the next TO problem.  

A more detailed description is reported in Algorithm~\ref{pseudocode}. In the initial phase (lines~\ref{line1}-\ref{line4}), the critic $V(x|\theta^V)$ and actor $\mu(x|\theta^\mu)$ networks (``\emph{Agent}'' blue block in Fig.~\ref{CACTO_scheme}) are initialized, as well as the target critic network $V'(x|\theta^{V'})$ (copying the weights of the critic), and a buffer $R$ is created. This buffer will store the transitions $(x_t,u_{\scriptscriptstyle{TO},t},\hat{V}_t,x_{\min(t+L+1,T)})$, where $\hat{V}_t$ is the partial cost-to-go until state $x_{\min(t+L+1,T)}$ after a rollout of either $L$ steps, or the number of steps to reach $T$ (whichever is lower). In this way, to update the critic one can choose between n-step Temporal Difference (TD) ($L=n-1$) and Monte-Carlo ($L=T-1-t$) by setting the hyperparameter $L$~\cite{TD_MC}. 
\setlength{\textfloatsep}{9pt}
\begin{algorithm}[t]
\caption{CACTO}\label{pseudocode}
\textbf{Inputs:} dynamics $f(\cdot,\cdot)$, running cost $l(\cdot,\cdot)$, terminal cost $l_T(\cdot)$, $T$, $M$, $S$, $K$, $L$ \label{line1} \\ 
\textbf{Output:} trained control policy $\mu(x)$ \\
\vspace{1mm}
$\theta^V \leftarrow$ random, $\theta^\mu \leftarrow$ random, $\theta^{V'} \leftarrow \theta^V$\\
Initialize replay buffer $R$ \label{line4}\\ 
\For{$episode\leftarrow 1$ \KwTo $M$}{\label{line5}
    \vspace{1mm}
    $x_0 \leftarrow$ random,  $x_{\scriptscriptstyle{TO},0}^{\circ} \leftarrow x_0$ \label{line6}\\
    \vspace{1mm}
    \For(\ \qquad\qquad\quad \ \ \emph{(policy rollout)}){$t\leftarrow x_0[n]$ \KwTo $T$}{
        \label{line7}
        \vspace{1mm}
        $u_{\mu,t}^{\circ} \leftarrow \mu(x_{\mu,t}^{\circ}|\theta^\mu)$\\
        \vspace{1mm}
        $x_{\mu,t+1}^{\circ} \leftarrow {\rm Environment}(x_{\mu,t}^{\circ},u_{\mu,t}^{\circ})$
        \vspace{1mm}
    }\label{line10}
    $(x_{\scriptscriptstyle{TO}}^{\circ},u_{\scriptscriptstyle{TO}}^{\circ}) \leftarrow (x_{\mu}^{\circ},u_{\mu}^{\circ})$\qquad\qquad\quad \ \ \emph{(TO warm-start)}\label{line11}\\
    \vspace{1mm}
    Solve TO problem and get control trajectory $u_{\scriptscriptstyle{TO}}$\\
    Agent's initial state $\leftarrow x_0$\\
    \For(\, \qquad\qquad\quad\emph{(episode rollout)}){$t\leftarrow x_0[n]$ \KwTo $T$}{\label{line14}
        \vspace{1mm}
        $x_{t+1}, l_t \leftarrow {\rm Environment}(x_t,u_{\scriptscriptstyle{TO,}}$$_t)$
        \vspace{1mm}
    }\label{line16}
    \For{$t\leftarrow x_0[n]$ \KwTo $T$}{\label{line17}
        Compute partial cost-to-go:
        \medmuskip=0mu
        \thinmuskip=0mu
        $\hat{V}_t = \sum\limits_{j=t}^{\min(t+L, T-1)} \ l_j$\\
         $R\leftarrow$$(x_t,u_{\scriptscriptstyle{TO}}$$_{,t}$$,\hat{V}_t,x_{\min(t+L+1,T)})$\\
    }\label{line20}    
    \If{$episode$ \% $e_{update}=0$}{
        \vspace{1mm}
        \For(\qquad\quad \ \, \emph{(critic $\&$ actor update)}){$k\leftarrow 1$ \KwTo $K$}{
            \vspace{1mm}
            Sample minibatch of $S$ transitions \label{line23} $(x_i,u_{\scriptscriptstyle{TO},i},\hat{V}_i,x_{\min(i+L+1,T)})$\\
            Compute cost-to-go: \label{line24}
            $\bar{V}_i = \begin{cases} 
            \hat{V}_i \quad &\text{if} \, i+L+1 > T \\
            \hat{V}_i + V'(x_{i+L+1}) \quad &\text{otherwise}
            \end{cases}$
            \vspace{1mm}
            Update critic by minimizing the loss over $\theta^V$: $L_c=\frac{1}{S}\sum\limits_{i=1}^{S} \left(\bar{V}_i - V(x_i|\theta^V) \right)^2 $\label{line25}\\
            Update actor by minimizing the loss over $\theta^\mu$: $L_a = \frac{1}{S}\sum\limits_{i=1}^{S}Q(x_i,\mu(x_i| \theta^{\mu}))$\label{line26}\\
            Update target critic: $\theta^{V'} \leftarrow \tau\theta^V+\left(1-\tau\right)\theta^{V'}$ \\ 
        }    
    }
}
\end{algorithm}\label{line30}
After the initialization phase, $M$ episodes are performed (lines~\ref{line5}-\ref{line30}) starting from a random initial state $x_0$ (line~\ref{line6}).
\oldtext{Since the state contains the time (as the last component of the state vector), the length of the episodes is variable.} 
\newtext{The state vector has a dimension of $n$ and includes the time as its last component. The starting time index for each episode (and partial cost-to-go computation) is based on $x_0[n]$ (lines~\ref{line14} and~\ref{line17}), so the length of each episode can vary.}
At the beginning of each episode, the state and control variables $(x_{\scriptscriptstyle{TO}}^{\circ},u_{\scriptscriptstyle{TO}}^{\circ})$ of a TO problem are initialized (line~\ref{line11}) with a rollout $(x_{\mu}^{\circ},u_{\mu}^{\circ})$ of the policy network $\mu(x|\theta^\mu)$ (lines~\ref{line7}-\ref{line10}, ``\emph{Warm-start}'' arrow in Fig.~\ref{CACTO_scheme}). The TO problem is then solved, the control inputs $u_{\scriptscriptstyle{TO}}$ are applied starting from $x_0$, and the resulting costs are computed and saved  (lines~\ref{line14}-\ref{line16}).
Then, for each step of the episode the partial cost-to-go $\hat{V}_t$ is computed and the related transition is saved in $R$ (lines~\ref{line17}-\ref{line20}, from ``\emph{Cost-to-go computation}'' block to ``\emph{Replay Buffer}'' one in Fig.~\ref{CACTO_scheme}). 
Finally, every $e_{update}$ episodes the critic and the actor are updated: for $K$ times a minibatch of $S$ transitions is sampled from $R$ (line~\ref{line23}, ``\emph{Sample mini-batch}'' arrow in Fig.~\ref{CACTO_scheme}), for each of them the complete cost-to-go $\bar{V}_i$ is computed \newtext{(line~\ref{line24})} by \newtext{either} adding the tail of the Value to $\hat{V}_i$ \newtext{\oldtext{(line~\ref{line24})}}\newtext{or copying the partial cost-to-go, which is the cost-to-go itself in case the lookahead window exceeds the episode length $T$. Then} the critic and actor loss functions are minimized (lines~\ref{line25}-\ref{line26}). Considering that CACTO deals with finite-horizon problems, it is worth noting that $V'$ is used only when $i+L+1\neq T$, namely when the costs-to-go are computed with n-step TD.
The critic and actor loss functions are respectively the mean squared error between the costs-to-go and the values predicted by the critic $\left( \bar{V}_i - V(x_i) \right)^2$ and the Q-value, namely $Q(x_t,u_t)=l(x_t,u_t)+V(x_{t+1})$, which represents the policy's performance.

\vspace{-5pt}
\subsection{Differences with respect to DDPG}
Our approach takes inspiration from DDPG, but presents a few key differences, which we highlight in this subsection. A first difference is that we replaced the \emph{Q-value} with the \emph{Value function}, which is approximated by the critic network. By doing so, the complexity is reduced since the critic's input is only the state, thus there is no need to explore the control space. Keeping a \emph{Q-value} approach instead, one could not explore the control space using only TO because only the locally-optimal control inputs would be chosen. Therefore, one should alternate the use of TO trajectories with other exploration techniques (\textit{e.g.}, acting greedily \textit{w.r.t.} the \emph{Q-value function} and adding some noise), which would dilute the benefits of our algorithm, making it more similar to standard RL. 
Contrary to DDPG, this algorithm is on-policy, meaning that the critic estimates the Value of the exploratory policy being followed. This is the policy obtained by initializing TO with rollouts of the current policy network. 
This implies that it is important to size $R$ not too large, so that only the most recent TO trajectories obtained by warm-starting TO with similar policies are stored; in this way, the critic, after its update, will approximate the \emph{Value function} associated to that policy. Otherwise, the risk is having a critic that represents a meaningless \emph{Value function}, because it used trajectories generated with very different policies.

Another difference with respect to DDPG is that we consider finite-horizon problems because TO cannot be used to solve arbitrary infinite-horizon problems. Therefore, the \emph{Value function} is time-dependent. We address this by considering time as the last component of the state vector ($x[n]=t$).

As in DDPG, we also make use of a target network $V'$ to improve the stability of our algorithm. It is a copy of the critic network, whose weights $\theta^{V'}$ are updated slower than the critic's ones $\theta^V$ by performing only partially the update of the critic, that is $\theta^{V'} \leftarrow \tau\theta^{V}+\left(1-\tau\right)\theta^{V'}$ with $\tau\ll1$ being the target learning rate.
\vspace{-5pt}
\subsection{Implementation Details}
Each TO problem was solved using collocation, which was  available in the \emph{Pyomo} library~\cite{pyomo} and solved with the nonlinear programming solver \emph{IpOpt}~\cite{ipopt}.

\newtext{The following details are not fixed as part of the proposed method. They are reported here for the sake of clarity and to help understand how the results in Section~\ref{Sec:results} have been obtained.} Concerning the critic's neural network, we used a structure with two small preprocessing fully-connected layers with respectively $16$ and $32$ neurons, followed by two other  fully-connected hidden layers with $256$ neurons each, all with \emph{ReLU} activation functions. For the actor instead we used a residual neural network with two fully-connected hidden layers with $256$ neurons each and \emph{ReLU} activation functions, whose outputs are added and passed as inputs to the last layer with a \emph{tanh} activation function to bound the final controls in $[-1,+1]$. We chose a residual neural network to better propagate the gradient to the first layer and limit the effects of the potential combination of the vanishing gradient problem due to the \emph{tanh} in the last layer and the dying ReLU problem~\cite{vanish_grad_dying_relu} which would prevent the network from continuing to learn. We finally multiply the output of this last layer by the control upper bound. To stabilize the training of these networks, the inputs are normalized and L2 weight and bias regularizers are used in each layer, with weight equal to $10^{-2}$.

The critic and actor loss functions were minimized with a stochastic gradient descent optimizer \textit{Adam}~\cite{adam}. We set the maximum number of episodes to $80000$ and stopped early the training of the neural networks if the results were satisfactory. 
\vspace{-20pt}
\subsection{\oldtext{Policy improvement}\newtext{Global convergence} proof for discrete spaces}
As for most continuous-space RL algorithms, it is hard to give any formal guarantee of convergence for CACTO.
However, we can show that, considering a discrete-space version of CACTO using look-up tables instead of DNN, the \oldtext{computed policy keeps improving at each iteration}\newtext{algorithm converges to a globally optimal policy}. 
\newtext{This version of CACTO performs sweeps of the entire state space as in classic Dynamic Programming (DP) algorithms (\textit{e.g.}, Policy Iteration~\cite{SuttonBarto}), rather than the asynchronous approach characterizing RL algorithms.
Moreover, we consider the Policy Iteration version of our algorithm meaning that each phase of policy evaluation and policy improvement converges before the other begins.}
This proof does not extend easily to the original CACTO algorithm, but it gives us an insight into the soundness of its key principle.
\oldtext{To keep the proof simple, we take the following assumptions.}
\setlength{\textfloatsep}{20.0pt plus 2.0pt minus 4.0pt}
\begin{theorem}
\newtext{Consider the following assumptions.}
\begin{itemize}
    \item \newtext{State and control spaces are finite.}
    \item \newtext{The optimal Value function $V^*$ is bounded.}
    \item \newtext{We have access to a discrete-space TO algorithm that can perform a local search and return a trajectory with a cost not greater than the cost of the initial guess.}
\end{itemize}
Let us define $\newtext{{}^k} \pi_{TO}$ as the control policy obtained solving a TO problem using a rollout of the policy $\newtext{{}^k} \pi$ as initial guess.
\newtext{Then, starting from an arbitrary initial policy ${}^0 \pi$, the following algorithm converges to the optimal Value function $V^*$ and an optimal policy $\pi^*$:
\begin{align*}
{}^k V^{\pi_{TO}} &= \textrm{PolicyEvaluation}({}^k \pi_{TO}) \\
{}^{k+1} \pi &= \argmin_{u \in \mathcal{U}} [ l + {}^k V^{\pi_{TO}}]
\end{align*}
}
\end{theorem}


\begin{proof}
Let us \oldtext{also }define $\pi'(x)$ as the policy obtained by \oldtext{choosing the control $u$}minimizing the \emph{Action-value function} of \oldtext{the TO policy}$\pi_{TO}$, namely $Q^{\pi_{TO}}(x,u)$\oldtext{, for all states $x$}:
\begin{equation}
    \label{pi_prime_def}
    \pi_t'(x_t) \triangleq \argmin_{u\newtext{ \in \mathcal{U}}} [ l_t(x_t,u) + V_{t+1}^{\pi_{TO}}(x_{t+1})] \quad \forall \, x_t, t
\end{equation}
where \newtext{$x_{t+1} = f_t(x_t, u)$, and} $V_{t+1}^{\pi_{TO}}(x_{t+1})$ is the cost-to-go following $\pi_{TO}$ from $x_{t+1}$ at time $t+1$.

Since TO always finds a solution that is at least as good as the provided initial guess,\oldtext{its Value cannot be greater than the Value of the policy used to compute the initial guess:}\newtext{we have that:}
%
\begin{equation}
    \label{VTO_Vpi}
    \newtext{V_t^{\pi}(x) \geq V_t^{\pi_{TO}}(x) \quad \forall \, x, t}
\end{equation}
\oldtext{Then, s}\newtext{S}ince $\pi'$ is the minimizer of $Q^{\pi_{TO}}$ (see \eqref{pi_prime_def}), we know that:
%
\begin{equation}
    \label{QTO_VTO}
    \newtext{V_t^{\pi_{TO}}(x) = 
    Q_t^{\pi_{TO}}(x,\pi_{TO,t}(x)) \geq
    Q_t^{\pi_{TO}}(x,\pi_t'(x))} 
\end{equation}
\oldtext{Thus,}\newtext{Starting from~\eqref{QTO_VTO}, and} following the same idea \oldtext{underlying}\newtext{of} the convergence proof of Policy Improvement~\cite{SuttonBarto}, we can write:
\begin{align}
    V^{\pi\newtext{_{TO}}}_t(x_t) & \geq \min_{u \newtext{\in \mathcal{U}}} [ l_t(x_t,u) + V_{t+1}^{\pi\newtext{_{TO}}}(x_{t+1})] \\
    \newtext{(\mathrm{by~\eqref{pi_prime_def}})} & \ \oldtext{\geq}\newtext{=} \ \oldtext{\min_{u}[}l_t(x_t,\pi_t'(x_t)) + V_{t+1}^{\pi_{TO}}(x_{t+1})\oldtext{]} \label{eq:l_s_pi_prime_V_TO}\\
    \newtext{(\mathrm{by~\eqref{QTO_VTO}})} & \geq l_t(x_t,\pi_t'(x_t)) + Q_{t+1}^{\pi_{TO}}(x_{t+1},\pi_{t+1}'(x_{t+1})) \ \oldtext{(\mathrm{by~\eqref{QTO_VTO}})} \label{eq:l_s_pi_prime_Q_TO} \\
    & =  l_t(x_t,\pi_t'(x_t)) + [l_{t+1}(x_{t+1},\pi_{t+1}'(x_{t+1})) 
    \label{eq:l_s_pi_prime_l_s_prime_pi_prime_V_TO} \\ 
    & \qquad + V_{t+2}^{\pi_{TO}}(x_{t+2}) ] \nonumber \\
    & \dots \nonumber && \\
    & \geq \sum_{k=t}^{T-1} l_k(x_k,\pi_k'(x_k)) + l_T(x_T) \triangleq V_t^{\pi'}(x_t) \label{eq:V_pi_prime} \\
    \newtext{(\mathrm{by~\eqref{VTO_Vpi}})} & \newtext{\ge V_t^{\pi'_{TO}}(x_t)}
\end{align}
To infer \eqref{eq:V_pi_prime} from \eqref{eq:l_s_pi_prime_l_s_prime_pi_prime_V_TO} we can iteratively apply the same reasoning we used to go from~\eqref{eq:l_s_pi_prime_V_TO} to \eqref{eq:l_s_pi_prime_l_s_prime_pi_prime_V_TO}. 
\oldtext{With this we have shown that the updated policy $\pi_t'(x)$ cannot be worse than the previous one $\pi_t(x)$, proving that the policy can only improve as the algorithm proceeds.}\newtext{This proves that the updated policy $\pi'_{TO}$ cannot be worse than the previous one $\pi_{TO}$. Since $V^{\pi_{TO}}$ is nonincreasing and always bounded from below by the optimal value $V^*$, it follows from the \emph{monotone convergence theorem} that $V^{\pi_{TO}}$ converges to a constant value $V^\infty$. At convergence we must have:
$$
V_t^\infty(x_t) = \min_{u \in \mathcal{U}} [ l_t(x_t,u) + V^\infty_{t+1}(x_{t+1})] \quad \forall \, x_t, t
$$
This is Bellman's optimality equation, which is a sufficient condition for global optimality, so it follows that the algorithm converges to the optimal Value ($V^\infty = V^*$) and to an optimal policy ($\pi^\infty = \pi^*$).
}
\end{proof}
\vspace{-5pt}
\section{Results}
\label{Sec:results}
This section presents the results of four different systems of increasing complexity\newtext{: single integrator, double integrator, Dubins car, and 3-joint planar manipulator. For each system, the task consists of finding the shortest path to a target point (related to the end-effector's position for the manipulator) while ensuring low control effort and avoiding an obstacle.} The aim is verifying the capability of CACTO to learn a control policy to warm-start TO so that it can find ``good'' trajectories, where other warm-starting techniques, such as using the initial conditions (ICS) or random values, would make it converge to poor local minima. \newtext{More precisely, the ICS warm-start uses the initial state (varying at each episode) as initial guess for the state variables in the OCP, for all time steps, and 0 as initial guess for the control variables.}  

For each system, we divided the XY-plane in a grid of $961$ points and, starting from each point with $0$ initial velocity, we compared the results of TO when warm-started with either the policy learned by CACTO, or a random initial guess, or the ICS for $x$ and $y$ (and 0 for the remaining variables). \oldtext{Table~\ref{tab2:comparison_warmstarting} reports the number of times that CACTO made TO find lower-cost solutions than the other two warm-starts, considering both the whole grid and the region from which it is harder for TO to find ``good'' solutions (\emph{Hard Region}).}\newtext{Table~\ref{tab1:time&Nupdates} reports the time and number of updates needed to train the critic and the actor for each test}.


\begin{figure}[t]
     \makebox[\columnwidth][c]{\includegraphics[trim={0 0.38cm 0 0.38cm},clip,width = 0.72\columnwidth]{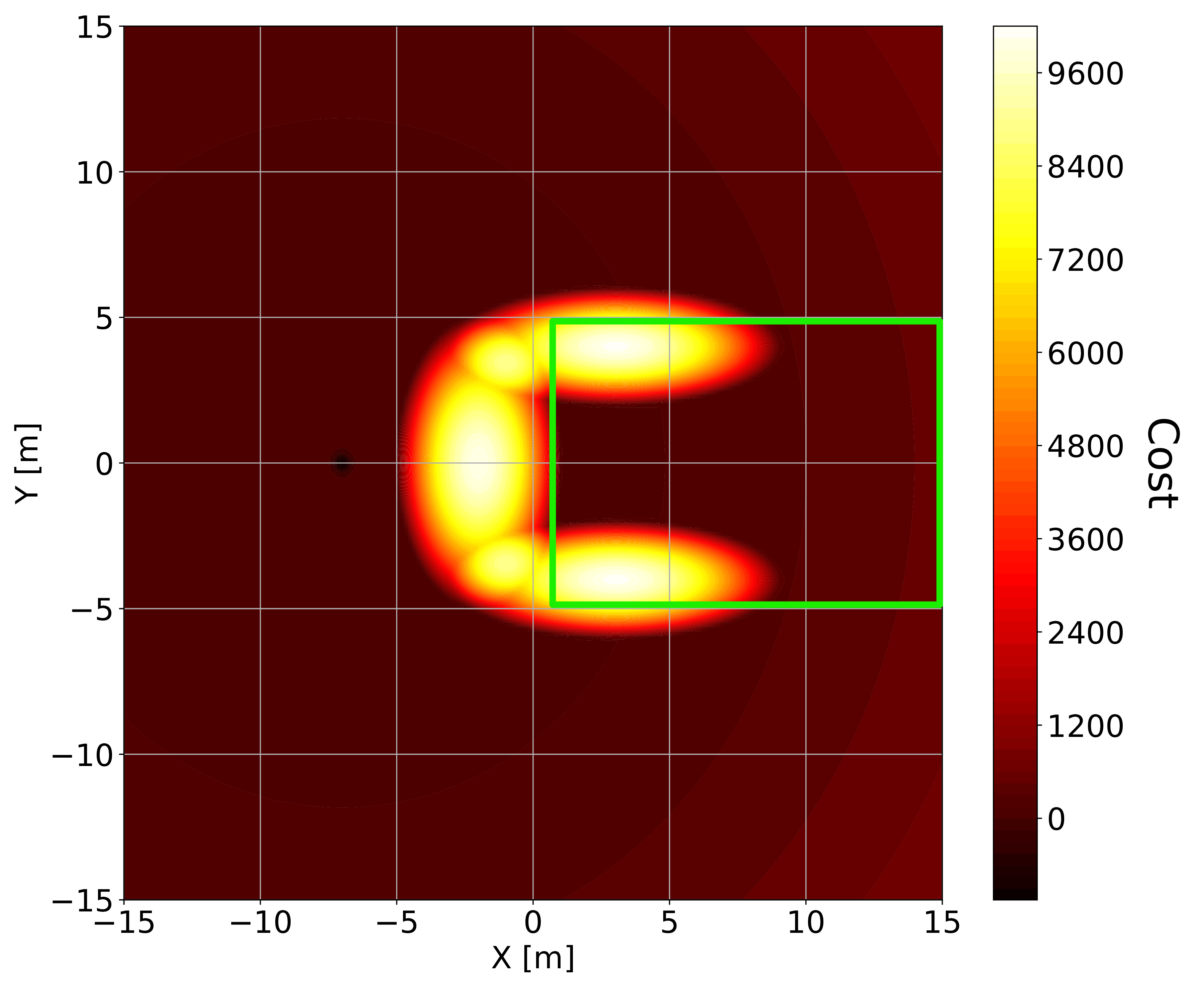}} 
    \caption{Cost function without the control effort term~\eqref{cost_term4}, considering a target point located at $[-7,0]$ with weights $w_d=100$, $w_p=5\cdot10^5$ and $w_{ob}=1\cdot10^6$. The green rectangle delimits the \emph{Hard Region}.}
    \label{fig_cost_function}
    \vspace{-5pt}
\end{figure}
In the four tests, we changed the dynamics of the system under analysis keeping the same highly non-convex cost function to ensure the presence of many local minima, that takes the following form:
\begin{align}
    & l(\cdot) = \frac{1}{c_2}\left(\sum_{i=1}^{4} l_i(\cdot)-c_1\right)  \label{cost_function}\\
    & l_1(\cdot) = w_{d}((x-x_{g})^2 + (y-y_{g})^2)  \label{cost_term1}\\
    & l_2(\cdot) = \frac{w_{p}}{\alpha_1}\ln(e^{-\alpha_1\left(\sqrt{(x-x_{g})^2 + c_2}+\sqrt{(y-y_{g})^2 + c_3} + c_4\right)}+1) \label{cost_term2}\\
    & l_3(\cdot) = \frac{w_{ob}}{\alpha_2}\sum_{i=1}^{3}-\ln{(e^{-\alpha_2\left(\frac{(x-x_{ob,i})^2}{(a_i/2)^2} + \frac{(y-y_{ob,i})^2}{(b_i/2)^2} - 1\right)} + 1)} \label{cost_term3}\\
    & l_4(\cdot) = w_{u}||u||_2^2  \label{cost_term4}
\end{align}
where $(x_g,y_g)$ are the target point coordinates, $c_1=10000$ and $c_2=100$ are two constants, while $c_3=0.1$, $c_4 = -2c_3-2\sqrt{c_3}$, $\alpha_1=50$ and $\alpha_2=50$ are the parameters that define the smoothness of the \emph{softmax} functions used to model respectively a cost valley in the neighborhood of the target and the three ellipses (centered at $(x_{ob,i},y_{ob,i})$ and with principal axes $a_i$ and $b_i$) forming the obstacle. The four terms composing the cost describe the task to be performed: \eqref{cost_term1} and \eqref{cost_term2} push the agent to reach the target point \newtext{(with weights $w_d$ and $w_p$, respectively)}, \eqref{cost_term3} makes it avoid the C-shaped obstacle represented by three overlapping ellipses \newtext{(with weight $w_{ob}$)} and \eqref{cost_term4} discourages it from using too much control effort \newtext{(with weight $w_u$)}. Fig.~\ref{fig_cost_function} illustrates the cost function without the control effort term~\eqref{cost_term4}, where the cost peaks correspond to the obstacle penalties~\eqref{cost_term3} and the cost valley to the term~\eqref{cost_term2}.
\oldtext{Table \ref{tab1:time&Nupdates} reports the time and number of updates needed to train the critic and the actor for each test.}
\begin{table}[!tb]
    \centering
    \caption{Time, number of DNN updates, number of environment steps, and learning rates ($LR_C$ for critic and $LR_A$ for actor) used for each test.}
    \renewcommand{\arraystretch}{1.3}
    \begin{tabular}{c|c|c|c|c|}
        & \textbf{\textit{Time [h]}}
        & \textbf{\textit{\# updates}}
        & \textbf{\textit{\# env. steps}}
        & $\boldsymbol{LR_C}$, $\boldsymbol{LR_A}$\\
        ~Single Int. &$5.46$ &$110k$ &$3.4M$ &$5e^{-3}, 1e^{-4}$\\
        \hdashline
        ~Double Int. &$7.29$ &$130k$ &$4.1M$ &$5e^{-3}, 5e^{-4}$\\
        \hdashline
        ~Dubins Car &$10.45$ &$260k$ &$4.1M$ &$1e^{-3}, 5e^{-4}$\\
        \hdashline
        ~Manipulator &$30.68$ &$385k$ &$6M$ &$1e^{-3}, 5e^{-5}$\\
    \end{tabular}
    \label{tab1:time&Nupdates}
    \vspace{-5pt}
\end{table}

The critic was updated with Monte-Carlo (MC) for the 2D point, and with 50-step TD for the car and the manipulator. The reason for using TD rather than MC lies in the stability of the training of critic and actor. Indeed, in the early phase of training, using Monte-Carlo could lead to large variations of the critic, and indirectly also of the actor, because the target Values would be extremely different from the current Values\oldtext{ computed by the critic network}. Moreover, in the early \newtext{training} phase\oldtext{ of training} TO could compute extremely poor trajectories because of the poor\oldtext{ quality of the} initial guess provided by the actor. In turn, these poor trajectories result in hard-to-learn \emph{Value functions}. Therefore, we have empirically observed that these two effects can destabilize the training, leading to either longer training times, or even divergence of the algorithm.

\newtext{To summarize the results, Table~\ref{tab2:comparison_warmstarting} reports the number of times that CACTO made TO find lower-cost solutions than the other two warm-starts, considering both the whole grid and the region from which it is harder for TO to find ``good'' solutions (\emph{Hard Region})}.
\begin{table*}[t]
    \centering
    \caption{Percentage of the time that warm-starting TO with CACTO leads to lower costs than using random initial guesses (CACTO vs. Random) and the initial conditions for x and y and 0 for the remaining variables (CACTO vs. ICS) as initial guess. Also the percentages when CACTO has lower or equal cost (\emph{i.e.}, including ties) as its competitor are reported. The best result out of 5 runs is reported for random warm-start. For each system, we sampled from a 31x31 grid for the initial x and y coordinates (those of the end effector for the manipulator) and setting the remaining initial state components to 0 (the joint positions of the manipulator were obtained by fixing the orientation of the end-effector and inverting the kinematics). The \emph{Hard Region} is the region delimited by $x\in[1,15]$ m ($[1,23]$ m in the manipulator test) and $y\in[-5,5]$ m.}
    \renewcommand{\arraystretch}{1.2}
    \begin{tabular}{c|cc|cc|}
        \multirow{3}{*}{\textbf{System}}
        &\multicolumn{2}{c|}{ \textbf{Whole Space}}
        &\multicolumn{2}{c|}{\textbf{{Hard Region}}} \\
        &\multicolumn{1}{c}{\textbf{\textit{CACTO $\boldsymbol{<}$ ($\boldsymbol{\leq}$) Random}} }
        &\multicolumn{1}{c|}{\textbf{\textit{CACTO $\boldsymbol{<}$ ($\boldsymbol{\leq}$) ICS}}}
        &\multicolumn{1}{c}{\textbf{\textit{CACTO $\boldsymbol{<}$ ($\boldsymbol{\leq}$) Random}}}
        &\multicolumn{1}{c|}{\textbf{\textit{CACTO $\boldsymbol{<}$ ($\boldsymbol{\leq}$) ICS}}}\\
        ~Single Int. &$99.88\%$ ($99.88\%$) &$14.49\%$ ($99.88\%$) &$99.11\%$ ($99.11\%$) &$91.96\%$ ($99.11\%$)\\
        \hdashline
        ~Double Int. &$99.88\%$ ($99.88\%$) &$12.38\%$ ($99.88\%$) &$99.11\%$ ($99.11\%$) &$91.96\%$ ($99.11\%$)\\
        \hdashline
        ~Dubins Car &$89.72\%$ ($98.83\%$) &$15.65\%$ ($95.56\%$) &$100\%$ ($100\%$) &$92.86\%$ ($100\%$)\\
        \hdashline
        ~Manipulator &$91.78\%$ ($91.91\%$) &$77.94\%$ ($78.33\%$) &$87.50\%$ ($87.50\%$) &$100\%$ ($100\%$)\\
    \end{tabular}
    \label{tab2:comparison_warmstarting}
    \vspace{-15pt}
\end{table*}
\begin{figure}[t]
    \subfloat[\scriptsize{ICS warm-start}]{\includegraphics[trim={14cm 0 15cm 0.41cm},clip,width = .5\columnwidth,height=3.9cm]{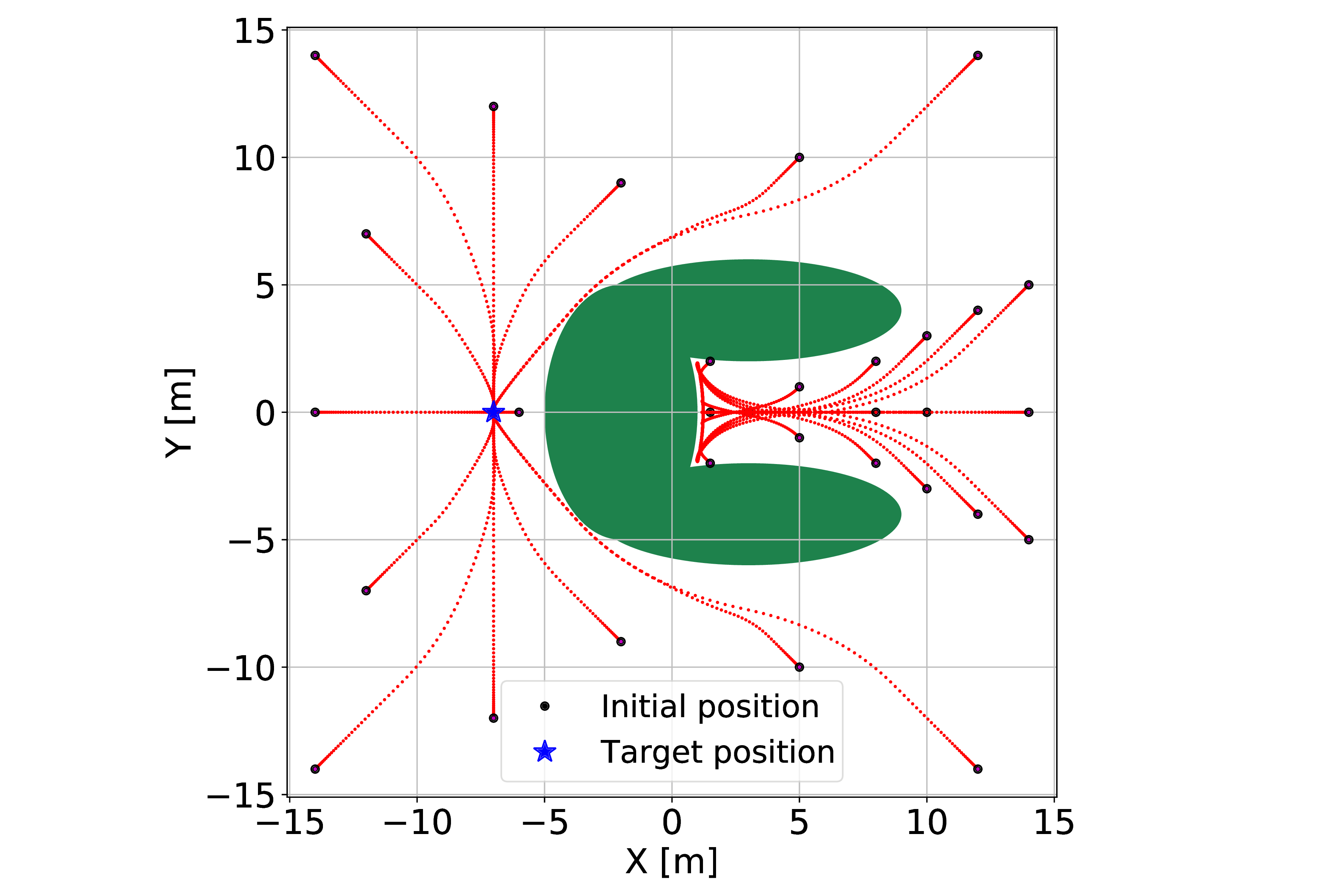}}\hfill 
    \subfloat[\scriptsize{CACTO warm-start}]{\includegraphics[trim={14cm 0 15cm 0.41cm},clip,width = .5\columnwidth,height=3.9cm]{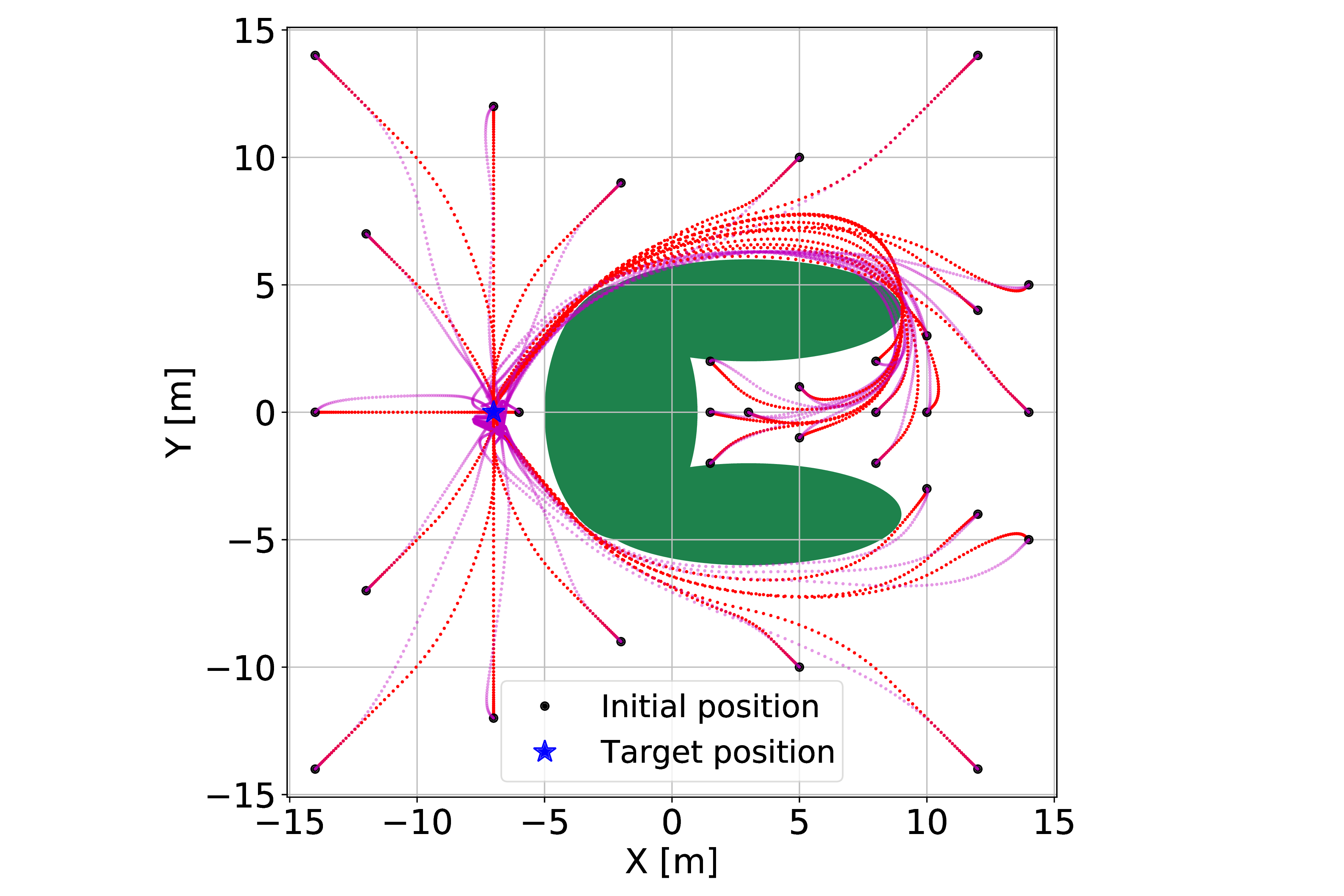}}
    \caption{Optimal trajectories (red) of the 2D double integrator obtained with ICS and CACTO warm-starts. In (b), the magenta lines represent the CACTO policy rollouts.}
    \label{TO_Int2D}
    \vspace{-5pt}
\end{figure}
\begin{figure}[t]
    \subfloat[\scriptsize{CACTO vs. Random warm-starts}]{\includegraphics[trim={0 0 0 1cm},clip,width = .49\columnwidth,height=3.8cm]{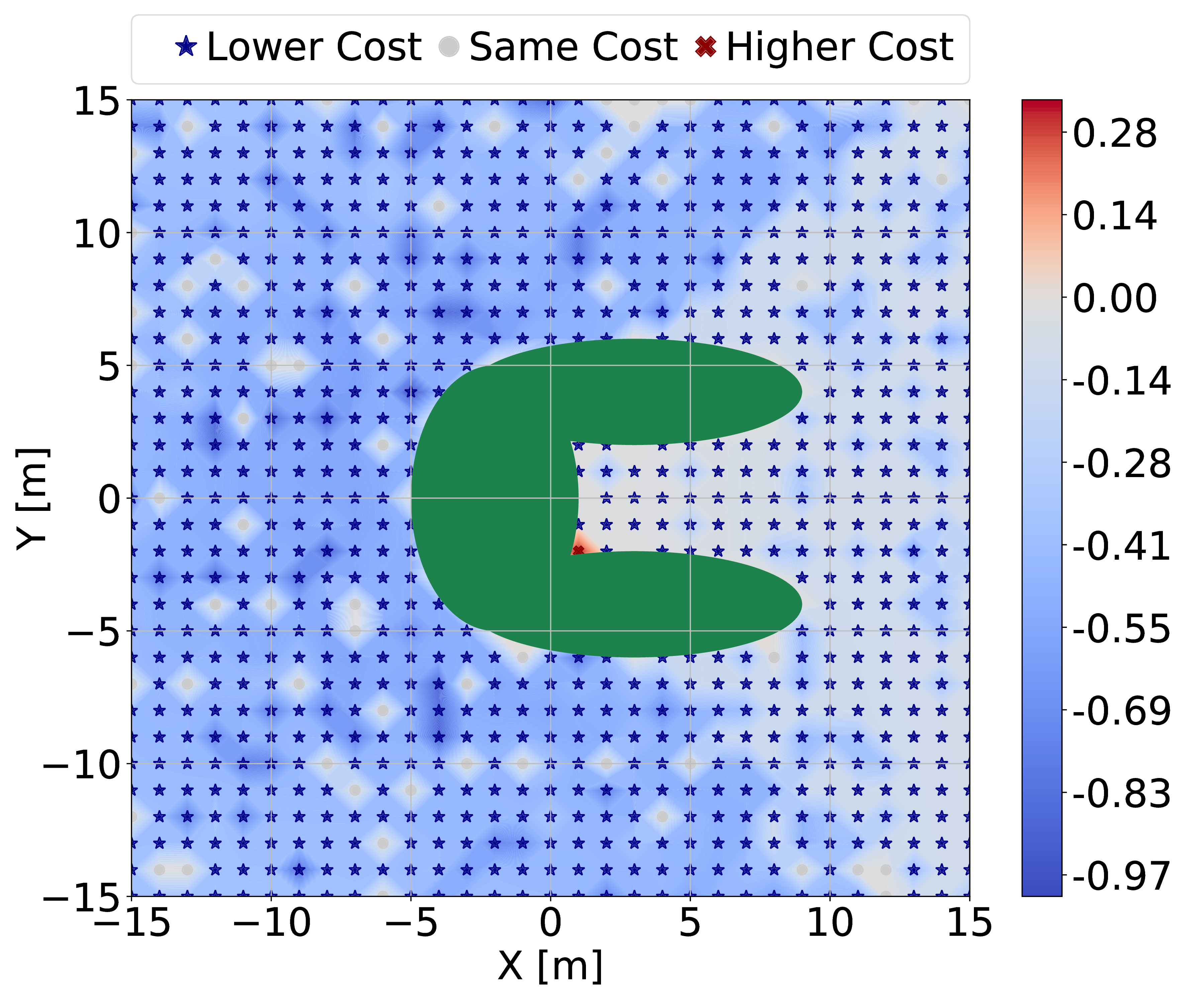}}\hfill 
    \subfloat[\scriptsize{CACTO vs. ICS warm-starts}]{\includegraphics[trim={0 0 0 1cm},clip,width =  .49\columnwidth,height=3.8cm]{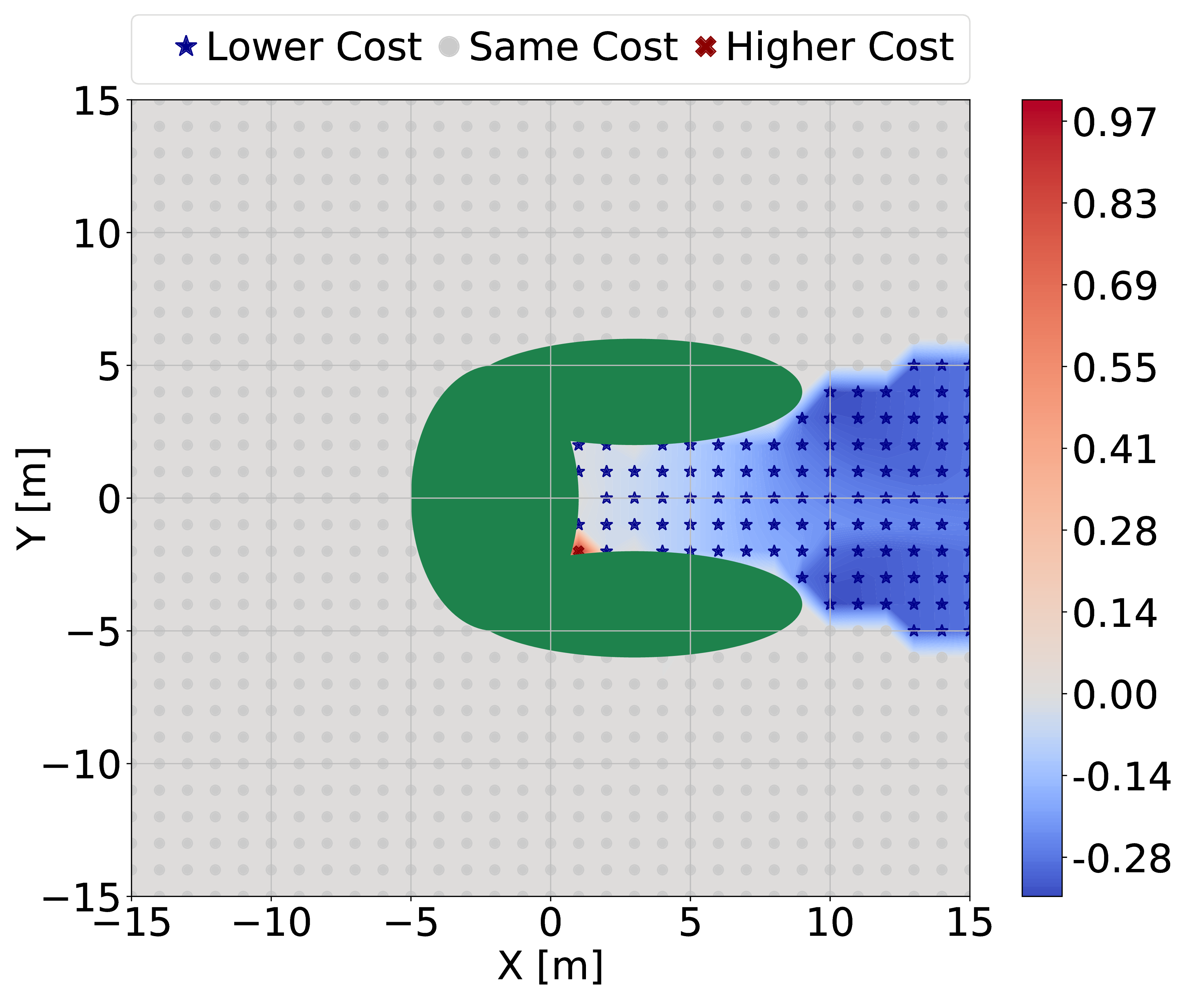}}
    \caption{Double integrator: cost difference between CACTO warm-start and other two warm-starts normalized by the largest cost difference. 
    }
    \label{DoubleInt2D}
    \vspace{-5pt}
\end{figure}
\vspace{-5pt}
\subsection{Single Integrator}
The first problem we considered is a simple 2D single integrator that has to reach a target point avoiding a C-shaped obstacle. The state is $[x,y,t] \in \mathbb{R}^3$, while the controls are the 2D velocities $[v_x,v_y] \in \mathbb{R}^2$, bounded in $[-4,4]\frac{m}{s}$.

The task is simple except when the system starts from the \emph{Hard Region} ($x\in[1,15]$ m and $y\in[-5,5]$ m), where TO can easily get stuck in local minima. Most of the times this means that the resulting trajectories point immediately towards the target, making the 2D point stay at the right boundary of the vertical ellipse. But since the C-shaped obstacle is modelled as three overlapping ellipses represented by three soft penalties~\eqref{cost_term3} in the cost function, it can also occur that the 2D point passes through them to reach the target, albeit at high cost.
Table.~\ref{tab2:comparison_warmstarting} reports the percentage of the time that warm-starting TO with CACTO rollouts leads to lower costs than using random values and the initial conditions for $x$ and $y$ and 0 for the remaining variables as initial guess. CACTO warm-start wins over the other two techniques, particularly if we consider the agent starting from the \emph{Hard Region} where warm-starting TO with CACTO leads to lower-cost solutions $99.11\%$ and $91.96\%$ of the time compared to using random values and the ICS as the initial guess, respectively.
\begin{figure}[t]
    \subfloat[\scriptsize{CACTO vs. Random warm-starts}]{\includegraphics[trim={0 0 0 0.35cm},clip,width = .49\columnwidth, height=3.8cm]{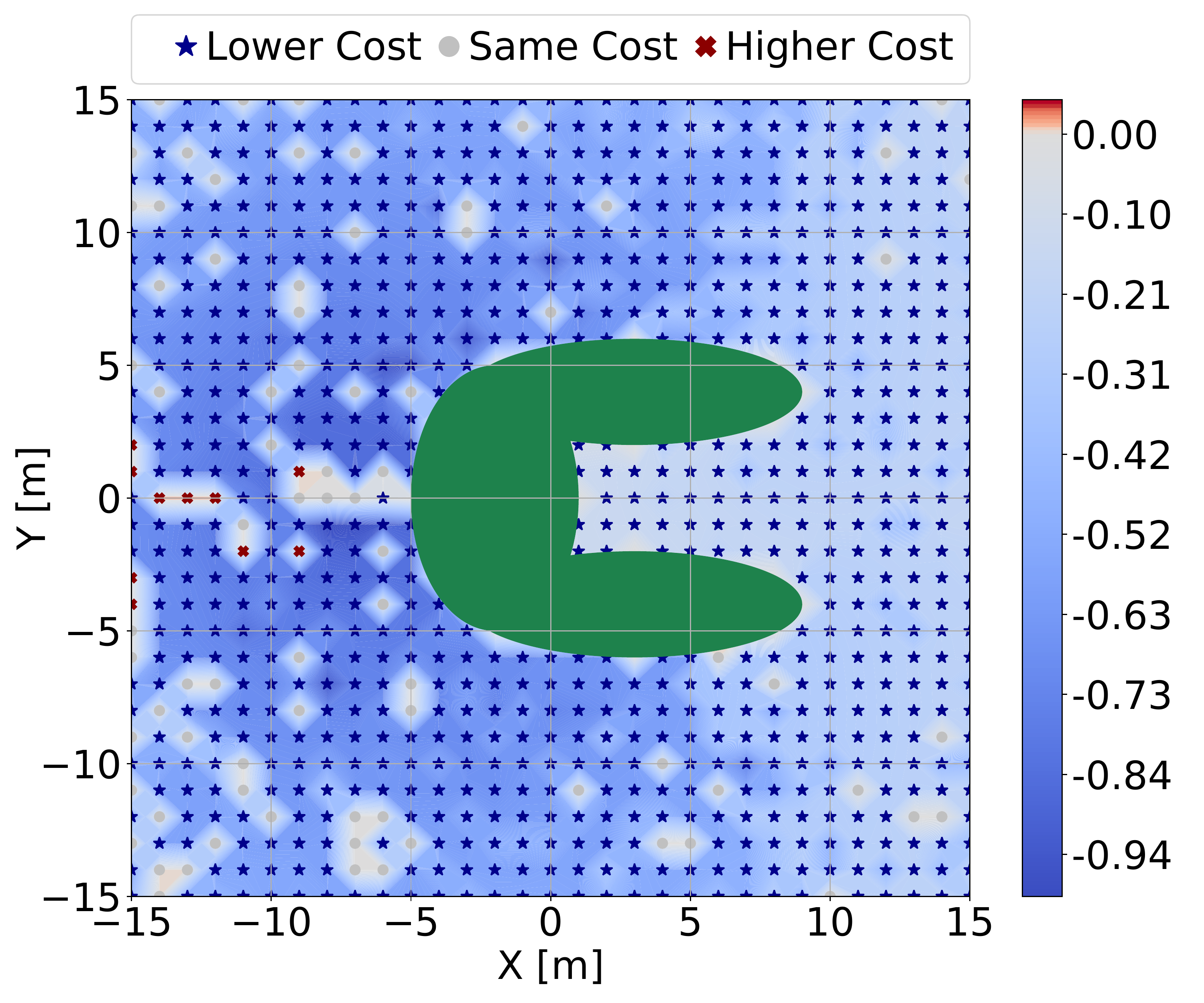}}\hfill 
    \subfloat[\scriptsize{CACTO vs. ICS warm-starts}]{\includegraphics[trim={0 0 0 0.35cm},clip,width = .49\columnwidth,height=3.8cm]{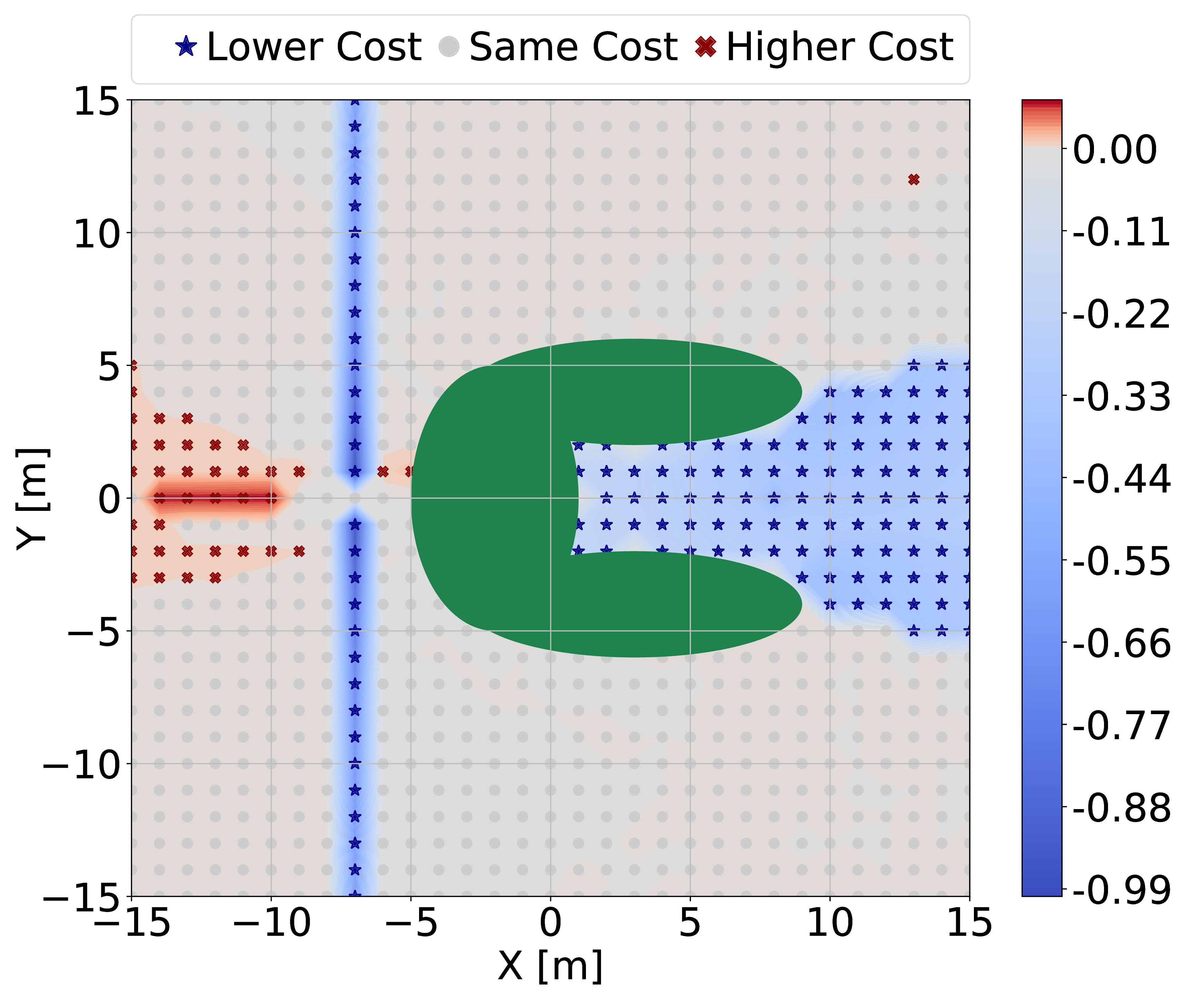}}
    \caption{Dubins car model: cost difference between CACTO warm-start and other two warm-starts normalized by the largest cost difference. 
    }
    \label{Car}
    \vspace{-7pt}
\end{figure}
\vspace{-5pt}
\subsection{Double Integrator}
\label{ssec:double_integrator}
We run the same experiment with the 2D point considering a double integrator dynamics. Therefore, now the state includes also the velocities $[x,y,v_x,v_y,t] \in \mathbb{R}^5$ and the control inputs are the accelerations $[a_x,a_y] \in \mathbb{R}^2$. As illustrated in Fig.~\ref{TO_Int2D}(b), the rollouts of the CACTO policy are already close to the globally optimal trajectories, therefore TO only needs to refine them when they are used as initial guess. \newtext{We refer to trajectories as globally optimal if their cost is the lowest among those obtained solving several TO problems warm-started with random initial guesses.}  
Fig.~\ref{DoubleInt2D} shows instead the cost difference (normalized by the highest difference in absolute value) when TO is warm-started with rollouts of the CACTO policy in place of random values or the ICS, respectively. Fig.~\ref{DoubleInt2D}(a) clearly shows that using rollouts of the policy learned by CACTO as an initial guess makes TO find lower-cost solutions from almost any initial state compared to those found with a random initial guess. In Fig.~\ref{DoubleInt2D}(b) instead, we can notice that ICS are a good initial guess for the majority of initial states and TO finds the same solutions as when warm-started with CACTO rollouts. However, when starting the agent in the \emph{Hard Region}, the ICS warm-start makes TO find poor local minima, where the agent remains stuck in that region or passes through the obstacle, whereas CACTO enables TO to successfully bring the agent to the target without touching the obstacle.
Also in this case, \oldtext{considering}\newtext{in} the \emph{Hard Region}, warm-starting TO with CACTO rollouts rather than with random values or ICS makes TO find lower costs with the same percentage as in the previous test.
\begin{figure}[t]
    \makebox[\columnwidth][c]{\includegraphics[trim={0 0.43cm 0 0.35cm},clip,width = 0.83\columnwidth,height=5cm]
    {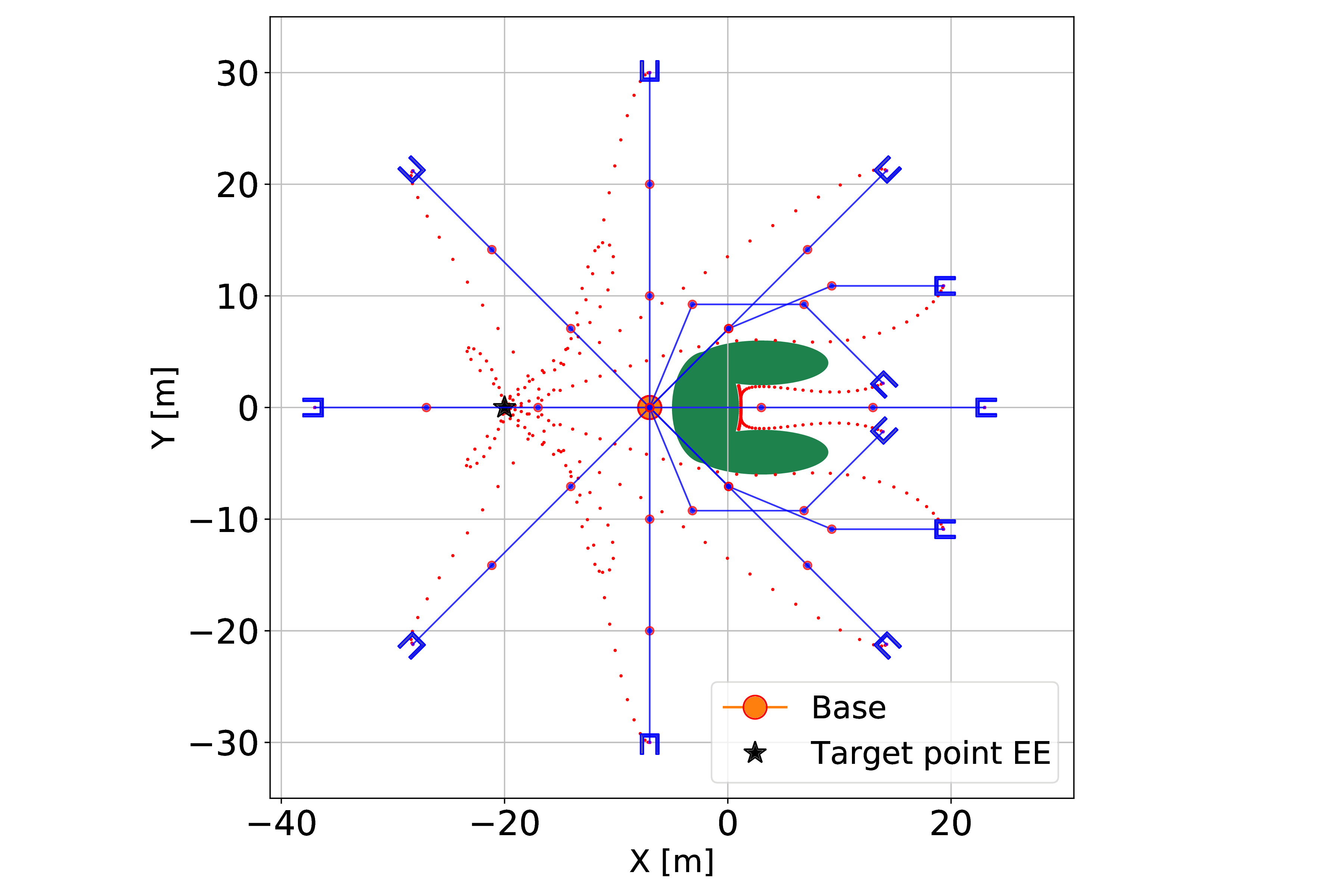}} 
    \caption{TO solutions considering 12 different initial configurations of the manipulator and ICS warm-start. The red dotted lines represent the trajectories performed by the end-effector (EE).}
    \label{TOmanipulator}
\end{figure}
\begin{figure}[t]
     \makebox[\columnwidth][c]{\includegraphics[trim={0 0.38cm 0 0.3cm},clip,width = 0.65\columnwidth ,height=5cm]{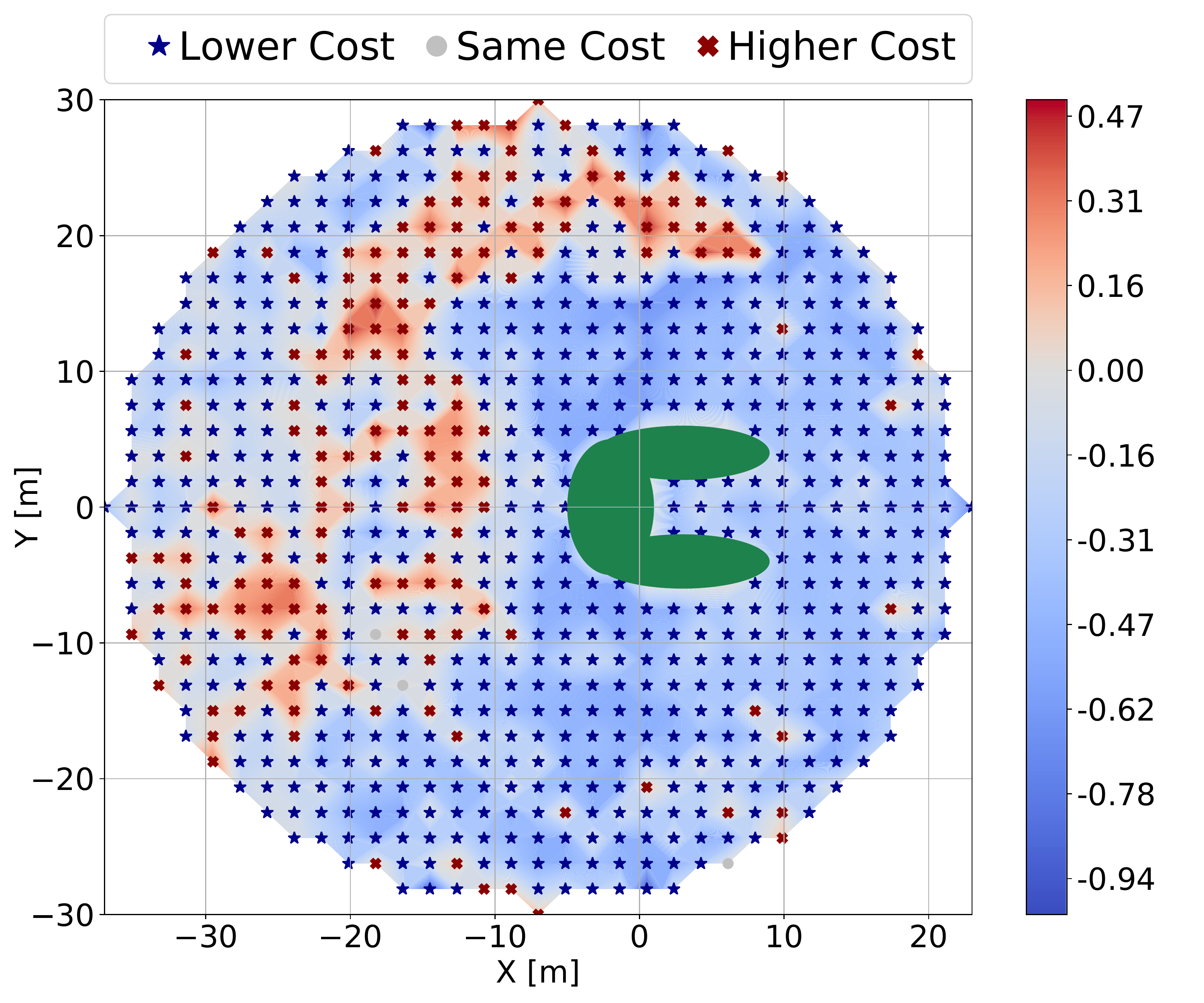}} 
    \caption{3-DoF Manipulator: normalized cost difference when warm-starting TO with CACTO compared to using the initial conditions as initial guess for the joint positions and 0 for the remaining variables. 
    }
    \label{ManipulatorICSinit}
    \vspace{-5pt}
\end{figure}
\vspace{-20pt}
\subsection{Dubins Car}
To test CACTO with a higher-dimensional system, we selected a jerk-controlled version of the so-called \emph{Dubins car model} \cite{dubins}, while keeping the same environment and cost function of the previous tests. Now the state has size $6$ because it includes the steering angle $\theta$, the tangential velocity $v$, and acceleration $a$, in addition to the coordinates $x$ and $y$ of the car's center of mass and time $t$: $s = (x,y,\theta,v,a,t) \in \mathbb{R}^6$. The control is still bi-dimensional and it consists in the steering velocity and the jerk $[\omega,j] \in \mathbb{R}^2$.
When the car starts from the \emph{Hard Region}, TO warm-started with CACTO always finds a lower cost than that obtained by warm-starting TO with \newtext{\oldtext{ICS or}} random values, \newtext{and it does so the $92.86\%$ of the time when compared to using ICS warm-start,} as reported in Table~\ref{tab2:comparison_warmstarting}.
In addition, Fig.~\ref{Car}(b) shows that TO warm-started with ICS is not able to find the globally optimal solution also when the car starts from a point along the vertical line passing through the target point. This is due to the fact that in that region the gradient of the cost is zero along the initialization itself.
%


%
%
\vspace{-5pt}
\subsection{3-DoF Planar Manipulator}
Finally, we tested our algorithm on a problem with a 7D state and 3D control space. It consists of a 3-DoF planar manipulator with base fixed at $[-7,0]$ m, working in the same environment of the previous tests, whose end-effector has to reach a target point located at $x_g = -20$ m and $y_g=0$ m. The cost function is always~\eqref{cost_function}, where $x$ and $y$ represent the coordinates of the end-effector. Fig.~\ref{TOmanipulator} shows some solutions found by TO when warm-started with the \oldtext{initial conditions}\newtext{ICS}. It may seem that TO succeeds in some cases in finding the globally optimal trajectories for those initial configurations, but actually all of them are only locally optimal, as shown by the negative cost difference in Fig.~\ref{ManipulatorICSinit} when those solutions are compared to the ones obtained by warm-starting TO with CACTO.
This test is harder, meaning that it is much easier for TO to find poor local minima, not only due to the larger state-action space, but also because the actor has to intrinsically learn the manipulator kinematics\oldtext{ to learn the control policy}. Indeed, considering the whole manipulator workspace, CACTO warm-start wins over using the ICS or random values $77.94\%$ and $91.78\%$ of the time, respectively, as reported in Table~\ref{tab2:comparison_warmstarting}.

%
\begin{figure}[t]
     \makebox[\columnwidth][c]{\includegraphics[trim={0 0cm 3.6cm 2.4cm},clip,width = \columnwidth]{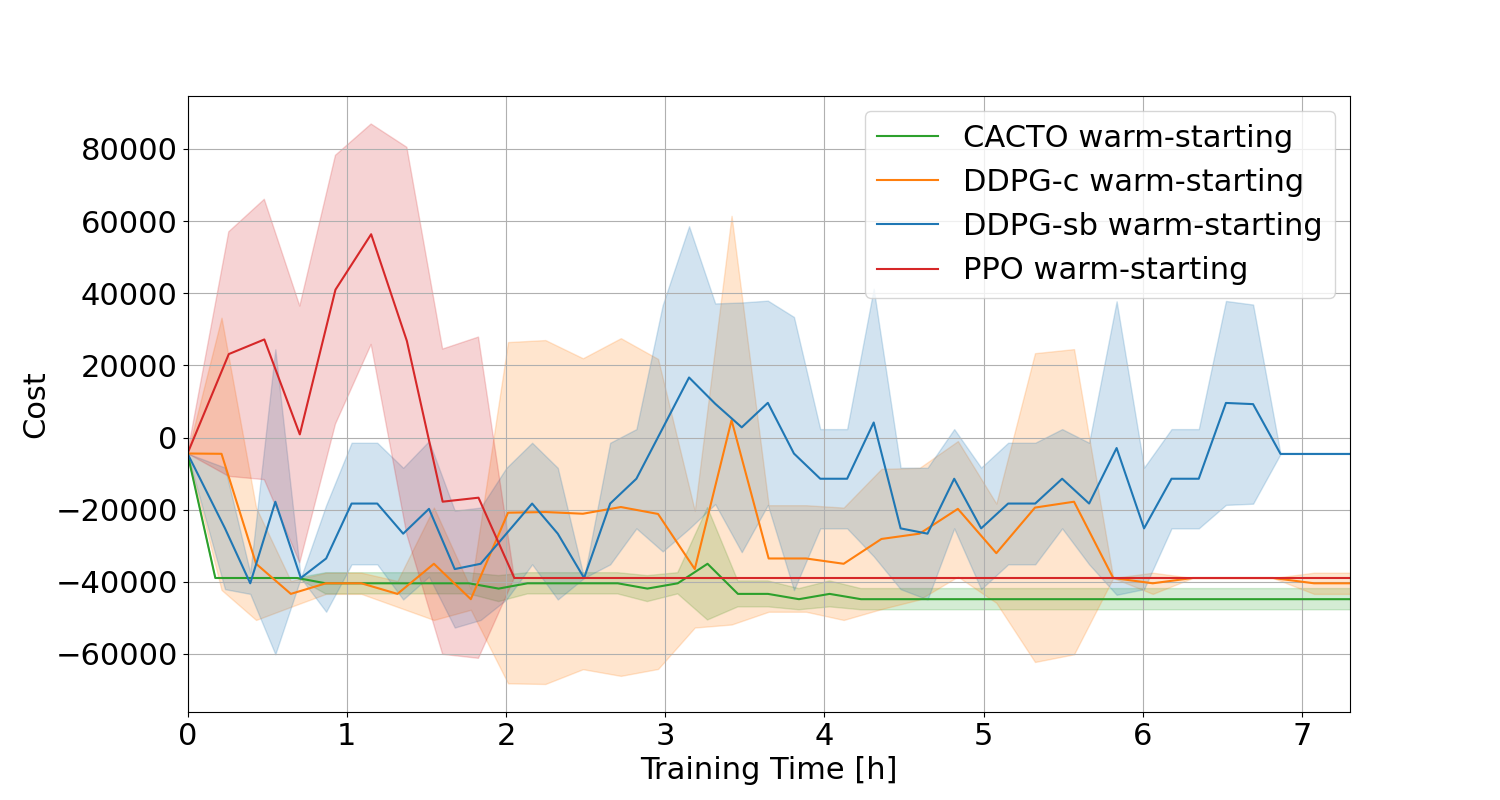}} 
    \caption{Cost of the trajectory found by TO (starting the 2D point from $[5,0]$m with $0$ velocity) when warm-started with rollouts of the policy learned with CACTO (\newtext{\oldtext{blue}}\newtext{green})\newtext{, our DDPG (DDPG-c in orange)}\newtext{, Stable Baselines' DDPG (DDPG-sb in blue)} and \newtext{\oldtext{DDPG}}\newtext{PPO} (red) as their trainings proceed. The shaded area denotes the standard deviation over 5 different runs.
    }
    \label{CACTOvsDDPG}
    \vspace{-5pt}
\end{figure}
\vspace{-5pt}
\subsection{Comparison with DDPG \newtext{ and PPO} performance}
To compare CACTO and DDPG, we have used our implementation of DDPG with hand-tuned hyper-parameters \newtext{(DDPG-c)}, \newtext{as well as that from Stable Baselines (DDPG-sb),} to find a warm-starting policy for the 2D double integrator test of Section~\ref{ssec:double_integrator}. \newtext{In addition, we have made the comparison also with PPO from Stable Baselines.}  Fig.~\ref{CACTOvsDDPG} shows the costs obtained by TO as functions of the computation time allocated to \newtext{\oldtext{CACTO and DDPG}}\newtext{each algorithm} to train \newtext{\oldtext{their}}\newtext{its} warm-starting polic\newtext{\oldtext{ies}}\newtext{y}. Clearly, CACTO was faster than \newtext{\oldtext{DDPG}}\newtext{the other RL algorithms} in learning a policy enabling TO to find lower-cost solutions, and its training was also more stable (rollouts with lower variance).
We have reported only the results with the double integrator because, despite our efforts, we did not manage to make DDPG converge in a reasonable number of episodes in the car and manipulator tests. Indeed, as stressed in \oldtext{[32]-[34]}\newtext{\cite{ddpg_instability1}}, DDPG is known to be very sensitive to its hyper-parameters making it prone to converge to poor solutions or even diverge. 

Besides this quantitative comparison, we could also try to qualitatively compare our results with the ones reported for DDPG in~\cite{ddpg}. Among their tests, the most similar to ours is the \textit{fixedReacher}, where a 3-DoF arm must reach a fixed target; this has the same dimensions as our manipulator test. However, our cost/reward function is highly non-convex, particularly due to the obstacle avoidance term, while theirs was quadratic and consisted of only two terms. This makes it harder to reach convergence in our DNN training. Using CACTO in such a simple setting, with a convex cost-function, would not make sense: TO would converge to the global optimum even with a trivial warm-start (indeed DDPG converged to roughly the same policy found by iLQG). Consequently, it makes no sense to use these results for a comparison with CACTO, which should instead be based on highly non-convex problems, where TO cannot find the global optimum with a naive initial guess. 
 

\section{Conclusions}
This paper presented a new algorithm for finding quasi-optimal control policies. In particular, we addressed the open problems affecting Trajectory Optimization and \oldtext{D}\newtext{d}eep Reinforcement Learning: the possibility of getting stuck in poor local minima when TO is not properly warm-started on one side, and the \oldtext{long training times (in addition to the strong dependence on the exploration process)}\newtext{low sample efficiency} of \oldtext{D}\newtext{d}eep RL. The proposed algorithm relies on the combination of TO and \oldtext{deep }RL in such a way that their interplay guides \oldtext{in an efficient way}\newtext{efficiently} the RL state-control space exploration process towards the globally optimal control policy, to be used then as TO initial guess provider.

We provided a \oldtext{proof of policy improvement}\newtext{global convergence proof} for a discrete-space version of our algorithm, which gives insight into its underlying theoretical principles.

We have shown the effectiveness of the algorithm testing it on four systems of increasing complexity, with highly non-convex cost functions, where TO struggles to find ``good'' solutions. Even though preliminary, our results validate our methodology and unlock a wide range of possible applications. 

Despite our encouraging results, we believe that CACTO can still be improved, in particular concerning its computation time. We are investigating how to speed up the learning of the critic using Sobolev Training~\cite{Sobolev}, which could exploit the derivatives of the \emph{Value function} computed by a DDP-like TO algorithm~\oldtext{[37]-[39]}\newtext{\cite{Crocoddyl}}. \newtext{We are also considering to implement CACTO using sampling-based multi-query planners, as done with PRMs in \cite{Jelavic}.} 
Moreover, besides using CACTO as initial guess provider for TO, where the RL and TO environments must match, we are interested in using it as a deep RL technique to find directly a control policy, where the two environments do not need to match (\emph{e.g.}, the TO problem could be a simplified version of the RL environment without noise sources).
Finally, we plan to extend CACTO to optimize also hardware parameters, to create a concurrent design (co-design) framework robust against the local minima problem, which is a crucial issue in co-design applications~\cite{redundant_actuation}.

\end{document}